\documentclass{article}

\usepackage[preprint]{neurips_2025}

\usepackage[utf8]{inputenc} 
\usepackage[T1]{fontenc}    
\usepackage{hyperref}       
\usepackage{url}            
\usepackage{booktabs}       
\usepackage{amsfonts}       
\usepackage{nicefrac}       
\usepackage{microtype}      
\usepackage{xcolor}         
\usepackage{amsmath}
\usepackage{amssymb}
\usepackage{graphicx}
\usepackage{amsthm} 
\newtheorem{theorem}{Theorem}[section]
\newtheorem{lemma}{Lemma}[section]

\newtheorem{corollary}{Corollary}

\title{Brewing Knowledge in Context: Distillation Perspectives on In-Context Learning}

\author{
  Chengye Li\thanks{Also with Key Laboratory of System Software (Chinese Academy of Sciences) and State Key Laboratory of Computer Science (Institute of Software, Chinese Academy of Sciences)} \\
  University of Chinese Academy of Sciences\\
  Beijing, China\\
  \texttt{licy@ios.ac.cn} \\
  \And
  Haiyun Liu\\
  University of South Florida\\
  Tampa, FL\\
  \texttt{haiyunliu@usf.edu}\\
  \And
  Yuanxi Li\\
  University of Illinois Urbana-Champaign\\
  Champaign, IL\\
  \texttt{li.yuanxi@outlook.com} \\
}

\begin{document}

\maketitle

\begin{abstract}
In-context learning (ICL) allows large language models (LLMs) to solve novel tasks without weight updates. Despite its empirical success, the mechanism behind ICL remains poorly understood, limiting our ability to interpret, improve, and reliably apply it. In this paper, we propose a new theoretical perspective that interprets ICL as an implicit form of knowledge distillation (KD), where prompt demonstrations guide the model to form a task-specific reference model during inference. Under this view, we derive a Rademacher complexity–based generalization bound and prove that the bias of the distilled weights grows linearly with the Maximum Mean Discrepancy (MMD) between the prompt and target distributions. This theoretical framework explains several empirical phenomena and unifies prior gradient-based and distributional analyses. To the best of our knowledge, this is the first to formalize inference-time attention as a distillation process, which provides theoretical insights for future prompt engineering and automated demonstration selection.

\end{abstract}

\section{Introduction}

Transformer-based large language models (LLMs) have recently reshaped natural-language processing (NLP) and had a broad impact on many related domains~\cite{brown2020language}.  
Built on the self-attention mechanism~\cite{vaswani2017attention}, models such as GPT series~\cite{radford2018improving}, BERT~\cite{devlin2019bert}, T5~\cite{raffel2020exploring}, and their larger successors~\cite{brown2020language,chowdhery2023palm,touvron2023llama,achiam2023gpt} demonstrate exceptional language modeling capacity and exhibit strong capabilities in cross-task generalization~\cite{radford2019language,bommasani2021opportunities}.
A key enabler of this generalization is in-context learning (ICL)~\cite{min2022metaicl,dong2024survey}, which allows LLMs to perform new tasks simply by conditioning on a few input–label examples, without any weight updates~\cite{liu2023pre,zhou2024mystery}. ICL can eliminate fine-tuning costs, keep examples human-readable, and endow LLMs with meta-learning-like adaptability~\cite{reynolds2021prompt,min2022metaicl}. These benefits have collectively contributed to the widespread adoption of ICL and a surge of research exploring its behavior, underlying principles, and practical applications~\cite{reynolds2021prompt,min2022rethinking,hahn2023theory}. Recent theoretical work frames ICL as learning a hidden predictor in a single forward pass.  
Specifically, some studies show that LLM outputs align with one step of gradient descent—under linear, softmax, and even with gradient-descent meta-optimisers~\cite{akyurek2023what,dai2023can,ren2024towards,von2023transformers}.  
Some works prove that attention can dynamically implement least-squares, Lasso, and related algorithms~\cite{bai2023transformers}, while a Bayesian line interprets ICL as implicit posterior inference and analyses its inductive bias~\cite{xie2022an,panwar2024incontext}.  
These works\footnote{Due to space limitations, more related works are discussed in Appendix~\ref{app:related_works}.} posit a latent model, but they leave open the question of how the corresponding weights emerge during inference.

While existing research has provided valuable insights, it remains largely descriptive and lacks a unified theoretical explanation of ICL’s behavior. Consequently, the underlying mechanism of ICL remains poorly understood~\cite{garg2022can,xie2022an}.

To address this gap, we revisit ICL through the lens of knowledge distillation (KD), a classical framework where a student model is explicitly trained to mimic a larger teacher~\cite{hinton2015distilling, RomeroBKCGB14}.  
We observe that ICL can be regarded as an ``implicit alignment”: when an LLM infers based on a few examples provided in the prompt, it performs a form of \emph{refinement} of its knowledge during the inference process and \emph{maps} this refined knowledge onto an implicit model. In other words, In-Context Learning can potentially be viewed as simulating a rapid distillation process during inference, extracting and aligning the abundant knowledge contained in the pretrained model with the task domain required by the current context.

This paper presents two main theoretical contributions. First, we demonstrate how the inference stage of the ``first prompt input” can be regarded as an implicit KD, condensing the pretrained model parameters into a ``reference model”. Second, within the framework of KD, we characterize how prompt design influences ICL performance.  When the prompt distribution closely aligns with that of the target task, it facilitates better knowledge transfer during inference. In contrast, mismatched prompts result in domain misalignment, degrading model performance due to ineffective distillation. These results establish a theoretical link  between ICL and KD, offering a novel theoretical perspective on how large models transmit knowledge under few-shot conditions, and providing guidance for designing more effective prompt strategies and distillation techniques.

\section{Preliminary}

\subsection{In-Context Learning: Definitions and Setup}
\label{subsec:icl-setup}

In-context learning (ICL) refers to the ability of a pretrained Transformer to perform new tasks by reading a small number of examples (\emph{demonstrations}) within the same input sequence that also contains \emph{queries} to be predicted. Unlike traditional fine-tuning, no parameter updates are performed. Below, we formally define the notation for demonstrations and queries, as well as the overall ICL process.

\paragraph{Demonstration and Query Tokens.}
Consider an input embedding dimension $d$. We represent the demonstration tokens (sometimes also referred to as ``context examples'') as the matrix
\(
  X_D \;=\; \bigl[x_1,\, x_2,\, \dots,\, x_N\bigr]
  \;\in\; \mathbb{R}^{d \times N},
\)
where each $x_i \in \mathbb{R}^d$ is a single demonstration token. Next, we denote the query tokens by
\(
  X_Q \;=\; \bigl[x'_1,\, x'_2,\, \dots,\, x'_M\bigr]
  \;\in\; \mathbb{R}^{d \times M}.
\)
These queries require predictions (e.g., completions or labels) that the model must infer by attending to $X_D$. We concatenate these two parts into a single input sequence
\(
  X \;=\; \bigl[X_D,\; X_Q\bigr]
  \;\in\; \mathbb{R}^{d \times (N + M)}.
\)
The Transformer takes $X$ as input in a forward pass, and the goal is to produce an appropriate output for the columns in $X_Q$, using only the information embedded in $X_D$ and $X_Q$ itself. In other words, any ``learning'' is accomplished purely by the model's internal attention and feed-forward mechanisms without modifying the pre-trained parameters.

\paragraph{ICL Process.}
Formally, in-context learning operates by first combining $N$ demonstration tokens and $M$ queries into a single sequence $X = [X_D, X_Q]$, then processing $X$ through a fixed pretrained Transformer to generate output sequence $H = \mathrm{Transformer}(X) \in \mathbb{R}^{d \times (N + M)}$.

Because no parameter update occurs and the model relies solely on attention to adapt to the $N$ demonstration tokens, it is said to learn ``in context.'' Thus, in-context learning can be regarded as a form of test-time adaptation based on few-shot examples, all embedded within a single input sequence.

\subsection{Connection Between ICL and Gradient Descent}
\label{subsec:icl_and_gd}

In this section, we highlight how in-context learning (ICL) can be viewed as performing a gradient-based update \emph{implicitly} during inference.

\begin{lemma}[Linear Attention as Gradient Descent in \cite{dai2023can}]
\label{lem:linear_attention_gd}
Suppose we have a linear layer 
\(f_L(x) = Wx\) with \(W \in\mathbb{R}^{d_o \times d_i}\), and a training set \(\{(x_i,\, y_i)\}_{i=1}^N\) with \(x_i \in \mathbb{R}^{d_i},\, y_i \in \mathbb{R}^{d_o}\).
Consider a gradient descent update on \(W\) using the predicted outputs \(\hat{y}_i = Wx_i\) and some differentiable loss \(\mathcal{L}(y_i, \hat{y}_i)\).
After a single gradient step with learning rate $\eta$, the updated weight matrix \(\widehat{W}\) can be decomposed into its initialization plus a training time correction as \(\widehat{W}=W_0 + \Delta W\) where \(\Delta W = -\eta\frac{\partial \mathcal{L}}{\partial W}\). Under certain mild conditions on the loss, the forward propagation process of a new test input is \emph{equivalent} to a linear-attention forward pass $\mathrm{LA}(V, K, q)$.
\end{lemma}
\begin{proof}
Full proof is provided in Appendix~\ref{app:lem:linear_attention_gd}.
\end{proof}
Therefore, \emph{one gradient update step} on a linear layer $W_0$ yields a test prediction identical to the \emph{linear-attention} forward pass, establishing the dual relationship between linear attention and gradient descent on a linear predictor.

\begin{lemma}[Softmax Attention as Gradient Descent (Theorem 3.1 in \cite{ren2024towards})]
\label{lem:softmax_attention_gd}
Recall that $X = [X_D,\;X_Q] \in \mathbb{R}^{d \times (N + M)}$ be a sequence of $N$ demonstration tokens $X_D \in \mathbb{R}^{d \times N}$ and $M$ query tokens $X_Q \in \mathbb{R}^{d \times M}$. 
Consider a single softmax-attention layer with trained parameters $(W^Q, W^K, W^V)$. For a new query token $x'_{M+1}$, the softmax-attention output
\(
h'_{\mathrm{M+1}}
\;=\;
W^VX\;\mathrm{softmax}\!\Bigl(\tfrac{(W^KX)^\top (W^Q\,x'_{\mathrm{M+1}})}{\sqrt{d}}\Bigr)
\;\in\;\mathbb{R}^d
\)
is equivalent to the test prediction $\hat{y}_{test}$ of the input $\phi(W^Qx'_{M+1})$ after one step of gradient descent on a reference model
\(
  f(x) \;=\; W\,\phi(x),
\)
where $\phi(\cdot)$ is a feature mapping function in the softmax kernel $K_{softmax}(x,y)=e^{x^\top y}$. More precisely, $h'_{\mathrm{M+1}}$ matches $f_{\mathrm{updated}}\!\bigl(W^Q\,x'_{\mathrm{M+1}}\bigr)$ after a single update on a self-supervised loss 
\[
\mathcal{L}=-\frac{1}{\eta D}\sum_{i=1}^N(W^Vx_i)^\top W\phi(W^Kx_i),
\]
where $\eta$ is learning rate and $D'$ is a constant.
\end{lemma}

\begin{proof}[Sketch of Proof]
\textbf{(1) Split Softmax Attention Operation.}
Recall that for a new query token $x'_{M+1}$, the softmax-attention output $h'_{\mathrm{M+1}}=W^VX\;\mathrm{softmax}\!\Bigl(\tfrac{(W^KX)^\top (W^Q\,x'_{\mathrm{M+1}})}{\sqrt{d}}\Bigr)$. Denote the core part of the softmax attention as
$A=\mathrm{softmax}\!\Bigl(\tfrac{(W^KX)^\top (W^Q\,x'_{\mathrm{M+1}})}{\sqrt{d}}\Bigr)$. Consider the \emph{Softmax} operation as \emph{Column Normalization} of a matrix, then have
\begin{align}
\label{eq:A_u}
A=A_uD^{-1}, \quad \text{where}\quad A_u=\exp\!\Bigl(\tfrac{(W^KX)^\top (W^Q\,x'_{\mathrm{M+1}})}{\sqrt{d}}\Bigr)\quad \text{and} \quad D=\text{diag}(\mathbf{1}_N^\top A_u).
\end{align}

\noindent
\textbf{(2) Kernel Approximation.}
\cite{ren2024towards} introduced an auxiliary kernel function called softmax kernel $K_{softmax}(x,y)=\exp(x^\top y)$. Notice that
\begin{align}
\label{eq:softmax_kernel}
K_{softmax}(x,y)=\exp(x^\top y)=\underbrace{\exp(||x||^2+||y||^2) }_{\text{a positive definite kernel}} \cdot \underbrace{\exp(-||x-y||^2)}_{\text{Gaussian Kernel}}
\end{align}
is also a positive definite kernel. Hence, according to Mercer's theorem~\cite{mercer1909xvi}, there exist some mapping function $\phi(\cdot)\colon \mathbb{R}^d \to \mathbb{R}^{r}$ satisfying $K_{softmax}=\phi(x)^\top \phi(y)$
to project the $\exp(x^\top y)$ into an inner product in a higher-dimensional Hilbert Space. If omitting the $\sqrt{d}$ and combine Eq.~\eqref{eq:A_u} with Eq.~\eqref{eq:softmax_kernel}, then we can have
\(
A_u=\exp\bigl((W^KX)^\top (W^Q\,x'_{\mathrm{M+1}})\bigr)={\phi(W^KX)}^\top\phi(W^Q\,x'_{\mathrm{M+1}}).
\)
Hence, the operations between tokens in softmax attention are equivalent to an inner product in a higher-dimensional space.

\noindent
\textbf{(3) Reference Model and Loss.}
Define a reference model
$f(x) \;=\; W \,\phi(x)$ with a self-supervised loss function 
\(
\mathcal{L}=-\frac{1}{\eta D}\sum_{i=1}^N(W^Vx_i)^\top W\phi(W^Kx_i), 
\)
where $\eta$ is learning rate and $D'$ is a constant. 

\noindent
\textbf{(4) One-Step Gradient Update.} 
Starting from an initialization $W_0$, a single gradient descent step on $\mathcal{L}(W)$ is
\(\widehat{W}=W_0 + \Delta W=W_0-\eta\frac{\partial \mathcal{L}}{\partial W}\).
For an input query token $x'_{M+1}$, let the input of $f$ is $W^Qx'_{M+1}$, then we have
\begin{equation}
\label{eq:hat_y}
\hat{y}_{test} = f(W^Qx'_{M+1}) = \;W_0\phi(W^Qx'_{M+1})+\frac{1}{D}\sum_{i=1}^N(W^Vx_i)^\top\phi(W_Kx_i)\phi(W^Qx'_{M+1})
\end{equation}

\noindent
\textbf{(5) Identifying the Softmax-Attention Output.} The output of a Softmax Attention Layer is $h'_{\mathrm{M+1}}=
W^VX\;\mathrm{softmax}\!\Bigl(\tfrac{(W^KX)^\top (W^Q\,x'_{\mathrm{M+1}})}{\sqrt{d}}\Bigr)$. Let $\widetilde{W}^K=\tfrac{W^K}{\sqrt{d}}$ and $\widetilde{W}^Q=\tfrac{W^Q}{\sqrt{d}}$, we have
\begin{equation}
\begin{aligned}
h'_{M+1} =  W^VX\;\text{softmax}\bigl((\widetilde{W}^KX)^\top\widetilde{W}^Qx'_{M+1}\bigr)=\;W^VX\;\phi(\widetilde{W}^KX)^\top\phi(\widetilde{W}^Qx'_{M+1})\cdot D^{-1},
\end{aligned}
\end{equation}
where $D^{-1}=\text{diag}\bigl(\mathbf{1}_N^\top\;\phi(\widetilde{W}^KX)^\top\phi(\widetilde{W}^Qx'_{M+1})\bigr)$. Consider an ICL process, let $X=[X_D,X_Q]$ and put all coefficient together as $D'$, then we have

\begin{equation}
\begin{aligned}
\label{eq:h'_{M+1}}
h'_{M+1} = \;\frac{1}{D'}W^VX_Q\phi(W^KX_Q)^\top\phi(W^Qx'_{M+1})+\frac{1}{D'}W^VX_D\phi(W^KX_D)^\top\phi(W^Qx'_{M+1}).
\end{aligned}
\end{equation}
Compare Eq.~\eqref{eq:hat_y} and Eq.~\eqref{eq:h'_{M+1}}, if we let $D=D'$ and consider that 
the embeddings of query tokens $X_Q$ as the initial weight $W_0$ and demonstration tokens provide the optimization signal, then we have
\(\hat{y}_{test}=h'_{M+1}\).
Hence, the ICL query output under softmax attention is \emph{exactly} the result of one gradient step in the reference model space, confirming the ICL--GD equivalence for softmax attention. Full proof is provided in Appendix~\ref{app:lem:softmax_attention_gd}.
\end{proof}
However, this derivation treats \emph{query tokens} as part of initializing the reference model's weights while using \emph{demonstration tokens} as a supervisory signal to provide training gradients for the reference model. This framework fails to adequately explain why large language models can generate well-generalized responses for query tokens after receiving only a single demonstration prompt input.

\section{In-Context Learning and Knowledge Distillation}
\label{sec:icl-kd}
In this section, we first formulate the formal definition of knowledge distillation, then establish the alignment relationship between reference models and knowledge distillation within the context of in-context learning. Finally, we propose mechanistic interpretations of how Transformer-based large language models operate during inference processes, analyzed through the lens of knowledge distillation principles.

\begin{corollary}
\label{cor:change}
Let the definitions of the two terms in Lemma \ref{lem:softmax_attention_gd} (Eq.~\eqref{eq:h'_{M+1}}) be interchanged: demonstration tokens are interpreted as providing the model’s initial weights $W_0$
 , while query tokens are treated as supplying gradient information $-\eta\frac{\partial\mathcal{L}}{\partial W}$ to optimize the model.
\end{corollary}
Based on Corollary \ref{cor:change}, we define the output of a Softmax Attention Layer as
\begin{align}
\label{new_hm+1}
h'_{M+1} = \underbrace{\frac{1}{D'}W^VX_D\phi(W^KX_D)^\top}_{W_0}\phi(W^Qx'_{M+1})+\underbrace{\frac{1}{D'}W^VX_Q\phi(W^KX_Q)^\top}_{-\eta\frac{\partial \mathcal{L}}{\partial W}}\phi(W^Qx'_{M+1}).
\end{align}
This allows us to argue that demonstration tokens \emph{implicitly initialize} a reference model (whose weight space may be infinite-dimensional). When the large language model encounters a new query token $x'_{M+1}$, the preceding query tokens $X_Q=\bigl[x'_1,\, x'_2,\, \dots,\, x'_M\bigr]$ function to fine-tune the model weights through gradient-based optimization. This mechanism thereby ensures that the LLM’s output aligns more closely with the prompt’s requirements.

\subsection{Knowledge Distillation} 
Knowledge distillation is a model compression technique that transfers knowledge from a pre-trained high-capacity \emph{teacher model} to a lightweight \emph{student model}. Instead of relying solely on ground-truth labels, the student learns from the teacher's softened probability outputs (called \emph{soft targets}), which contain richer dark knowledge \cite{hinton2015distilling}.

Let $\mathcal{T}\colon\mathcal{X}\!\to\!\mathbb{R}^{d}$ be a \emph{teacher} model that has been fully pre‑trained before \emph{any} tokens of the current task are observed, and let $\mathcal{S}(\,\cdot\,;W)\colon\mathcal{X}\!\to\!\mathbb{R}^{d}$ be a \emph{student} model with trainable parameters $W\!\in\!\mathbb{R}^{d\times m}$.  
Given a \emph{distillation set} $\mathcal{D}\!=\!\{x_i\}_{i=1}^{N}$ sampled from an input distribution $P_{\!X}$, KD seeks the parameter matrix
\[
  W^{\star}\;=\;\mathop{\arg\min}\limits_{W}\;
  \mathcal{L}_{\mathrm{KD}}(W),\qquad
  \mathcal{L}_{\mathrm{KD}}(W)
  \;=\;
  \mathbb{E}_{x\sim P_{X}}
  \Bigl[
    \operatorname{Dist}\bigl(\mathcal{T}(x),\mathcal{S}(x;W)\bigr)
  \Bigr],
\]
where $\operatorname{Dist}(\cdot,\cdot)$ is a task‑dependent dissimilarity measure, $\mathcal{T}(x)$ is teacher signal and $\mathcal{S}(x;W)$ is student prediction.  
Minimizing $\mathcal{L}_{\mathrm{KD}}$ transfers the knowledge encoded in the fixed teacher parameters into the student weights $W^{\star}$, yielding a compact model whose behavior on $\mathcal{D}$ (and, by generalization, on unseen inputs) approximates that of the teacher.

\subsection{Weight Initialization as Distillation}

Building upon Eq.~\eqref{new_hm+1}, we reveal that the initial weight matrix $W_0$ of the reference model exhibits intrinsic connections with the primitive input embeddings of Large Language Models (LLMs).  This observation naturally leads us to consider the conceptual framework of knowledge distillation (KD).  We propose to align the initialization protocol of $W_0$ with KD principles, thereby providing more granular theoretical insights into the operational dynamics of the reference model.

\begin{theorem}[Reference Model Initialization as Knowledge Distillation with $l_2$ Distance]
\label{the:w0_and_kd}
Under the framework of in-context learning with softmax-attention-based Transformers, consider a reference model $f(x)=W\phi(x)$ where $\phi(\cdot)$ denotes the feature kernel mapping. The weight initialization of $W$ is $W_{init}$, formally equivalent to a knowledge distillation process, wherein the demonstration inputs $X_D$ implicitly induce the pre-trained LLM to distill knowledge derived from its parameter $\theta_{LLM}$ into $W_{init}$. In this KD process, the teacher model $f_T(x;\theta_{LLM})$ corresponds to the pre-trained LLM with frozen parameters. The student model $f_S(x; W)=W\phi(x)$ represents the reference model with learnable weights. The distillation loss $\mathcal{L}_{\text{KD}}(W)$ measures output similarity via:
\begin{equation}
\label{eq:theorem_ICL-KD_loss}
\mathcal{L}_{\text{KD}}(W) = \mathbb{E}_{x \sim \mathcal{D}} \bigl[||f_T(x;\theta_{LLM}) - f_S(x;W)||^2\bigr]
\end{equation}
\end{theorem}

\begin{proof}[Sketch of Proof] Let $W_{init}$ be the initial weight matrix. In the in-context learning scenario, the distillation set is $\mathcal{D}^*=X_D$, and the knowledge of the teacher model is implicitly encoded in the pre-trained parameters $W^V, W^K$. The teacher output is defined as $f_T(x_i)=W^V x_i$. The student model is defined as $f_S(x) = W\phi(W^K x)$. Then we minimize the output discrepancy between student and teacher models on demonstration tokens:
\begin{equation}
\begin{aligned}
\mathcal{L}_{\text{KD}}(W) =&\; \mathbb{E}_{x \sim \mathcal{D}^*} \bigl[||f_T(x_i) - f_S(x)||^2\bigr]=\frac{1}{N}\sum_{i=1}^N \left\| W\phi(W^K x_i) - W^V x_i \right\|^2  
\end{aligned}
\end{equation}

The gradient of $L_{\text{KD}}$ with respect to $W$ is computed as:
\begin{equation}
\frac{\partial \mathcal{L}_{\text{KD}}}{\partial W} = \frac{2}{N}\sum_{i=1}^N \left[ W\phi(W^K x_i) - W^V x_i \right] \phi(W^K x_i)^\top
\end{equation}

Performing single-step gradient descent from initial $W_{\text{init}} = \mathbf{0}$ with learning rate $\eta^*$:
\begin{equation}
\label{eq:W_*}
W^* = W_{\text{init}} - \eta^* \left. \frac{\partial L_{\text{KD}}}{\partial W} \right|_{W=0} = \frac{2\eta^*}{N} \sum_{i=1}^N W^V x_i \phi(W^K x_i)^\top
\end{equation}

In Eq.~\eqref{new_hm+1}, $W_0$ of the reference model is $\frac{1}{D'}W^VX_D\phi(W^KX_D)^\top$. Obviously, when let $\frac{2\eta^*}{N}=\frac{1}{D'}$ we can have:
\begin{align}
\label{eq:W^*_and_W_0}
W^* =  \frac{2\eta^*}{N} \sum_{i=1}^N W^V x_i \phi(W^K x_i)^\top =  \frac{2\eta^*}{N}W^V X_D \phi(W^K x_D)^\top = W_0
\end{align}
Thus, the proof is completed.
\end{proof}
Recall that $D^{-1}=\text{diag}\bigl(\mathbf{1}_N^\top\;\phi(\widetilde{W}^KX)^\top\phi(\widetilde{W}^Qx'_{M+1})\bigr)$, let $D^{-1}=\frac{1}{D'}I_N$, then we can have
\begin{align}
D'=\mathbf{1}_N^\top\phi(\widetilde{W}^KX)^\top\phi(\widetilde{W}^Qx'_{M+1})=\sum_{i=1}^{N+M}\exp\!\bigl((W^{K}x_i)^{\top}(W^{Q}x'_{M+1})/\sqrt d\bigr),
\end{align}
So, the factor $\frac{1}{D'}$ in Eq.~\eqref{new_hm+1} is the \emph{partition function} (normalization constant) of the softmax attention, $\frac{1}{D'}$ rescales the un‑normalized scores to a probability simplex.
When only a few keys are highly similar to the query, the summation $D'$ is small, and the resulting distribution is sharply peaked; conversely, many comparable similarities make $D'$ large and lead to a flatter, higher‑entropy distribution.

Notice that $\eta\propto\frac{1}{D'}$, this implies: For an input token sequence exhibiting domain stationarity (which means flat attention distribution), the model’s attention is diffuse, with similar attention scores across all input tokens. This corresponds to a large $D'$ value, which yields a smaller $\eta$ and consequently gentler gradient updates during reference model initialization from the knowledge distillation perspective.
Conversely, a weaker normalization intensity (smaller $D'$) reflects sharpened attention focus, where a subset of input tokens receives disproportionately higher attention scores than others. This induces a larger $\eta$, necessitating more aggressive knowledge distillation to align the reference model.

Overall, we propose that in the \emph{implicit Knowledge Distillation} of In-Context Learning, large language models generate \emph{value} responses to input demonstration tokens $x$ (referred to as prompts), functioning as the \emph{teacher} model. The \emph{student} model, correspondingly, constitutes a mapping of \emph{key} responses to input $x$ within a (potentially infinite) high-dimensional latent space. This implies that during the distillation process, weight initialization represents the distillation of pretrained LLMs' prompt-induced value responses into key tokens as compact representations. Consequently, when encountering novel \emph{query} tokens, their interaction with key tokens inherently performs implicit computations with the complete value space of the original pretrained LLM. This framework aligns seamlessly with the research community's established understanding of Query-Key-Value mechanisms.

\subsection{Generalization error bound in the ICL-KD procedure}
Nevertheless, performance degradation is inevitably incurred during the knowledge distillation process. We aim to establish a generalization error bound for in-context learning from the knowledge distillation perspective. First, we give several instrumental lemmas.

\begin{lemma}[Talagrand's Contraction Lemma \cite{ledoux2013probability,bartlett2002rademacher}]
\label{lemma:talagrand}
Let $\mathcal{F}\subseteq \mathbb{R}^\mathcal{X}$ be a class of real-valued functions, and let $\psi:\mathbb{R}\to\mathbb{R}$ be a Lipschitz-continuous function with $\psi(0)=0$ and Lipschitz constant $L$, i.e., $|\psi(a)-\psi(b)| \le L \cdot |a-b|$ for all $a,b\in \mathbb{R}$. Then the Rademacher complexity of the composed class $\psi \circ \mathcal{F}$ satisfies
\(
  \mathcal{R}_N(\psi \circ \mathcal{F}) \; \le \;L \,\mathcal{R}_N(\mathcal{F}).
\)
\end{lemma}
\begin{proof}
    Proof is provided in Appendix~\ref{app:talagrand}.
\end{proof}

\begin{lemma}[Rademacher Complexity of Linear Operators \cite{bartlett2002rademacher}]
\label{lemma:rademacher-linear}
Consider the function class
\(
  \mathcal{H} \;=\; \Bigl\{\; h(W,x) \;=\; W\,\phi(x)\;\Big|\;\|W\|_F \,\le\, B\Bigr\},
\)
where $\phi:\mathcal{X}\to \mathbb{R}^r$ satisfies $\|\phi(x)\|\;\le\;C$ for all $x$. Then its Rademacher complexity is bounded by
\(
  \mathcal{R}_N(\mathcal{H}) \;\le\; \frac{B\,C}{\sqrt{N}}.
\)
\end{lemma}
\begin{proof}
    Proof is provided in Appendix~\ref{app:rademacher-linear}.
\end{proof}

\begin{theorem}[Generalization Bound for Implicit Knowledge Distillation]
\label{thm:generalization-bound}
Suppose the following conditions hold:
\textbf{(i)} The teacher model is given by $f_T(x) \;=\; W^V x$ with $\|W^V x\| \,\le\, D$ for all $x$;
\textbf{(ii)} The student model is $f_S(x;W) = W\,\phi\bigl(W^K x\bigr)$, where $\|W\|_F \,\le\, B$ and $\|\phi(W^K x)\| \,\le\, C$ for all $x$.  Here, $W^K$ is the fixed weight from LLM;
\textbf{(iii)} the demonstration samples $\{x_i\}_{i=1}^N$ are drawn i.i.d.\ from $P_{X_D}$.
Then for any $\delta>0$, with probability at least $1-\delta$, the expected risk $L(W)$ and the empirical risk $\hat{L}(W)$ satisfy
\begin{equation}
\label{eq:generalization_bound}
  L(W)
  \;\;\le\;\;
  \hat{L}(W)
  \;+\;
  \frac{4\,B\,C\;\bigl(D + B\,C\bigr)}{\sqrt{N}}
  \;+\;
  3\;\bigl(D + B\,C\bigr)^2 \,\sqrt{\frac{\log\!\bigl(2/\delta\bigr)}{2\,N}}. 
\end{equation}
\end{theorem}
\begin{proof}[Sketch of proof]
\noindent
\emph{Step 1: Define the loss function class.}\\
Let the single-sample loss (same as Eq.~\eqref{eq:theorem_ICL-KD_loss}) be
\(
  l(W,x) \;=\; \bigl\|\,f_T(x) \;-\; W\,\phi\bigl(W^K x\bigr)\bigr\|^2,
\)
and define the hypothesis set
\(
  \mathcal{L} \;=\; \Bigl\{\;l(W,\cdot)\,\Big|\,\|W\|_F \le B\Bigr\}.
\)

\noindent
\emph{Step 2: Decompose the loss function.}\\
Set 
\(
  h(W,x) =f_T(x)-W\phi\bigl(W^K x\bigr), 
\)
so $l(W,x) =\|h(W,x)\|^2.$ Let
\(
  \mathcal{H} =\bigl\{h(W,\cdot)\bigm|\|W\|_F \le B\bigr\},
\)
and observe that $\mathcal{L} \;=\; \psi \circ \mathcal{H}$ with $\psi(z)=\|z\|^2$.

\noindent
\emph{Step 3: Verify the conditions for Talagrand's contraction.}\\
For any $z_1,z_2 \in\mathbb{R}^d$, we have
\(
  \bigl|\psi(z_1)-\psi(z_2)\bigr|
  \;=\;
  \bigl|\|z_1\|^2 - \|z_2\|^2\bigr|
  \;\le\;
  (\|z_1\|\,+\,\|z_2\|)\;\|z_1 - z_2\|.
\)
Since
\(
  \|h(W,x)\| \;\le\; \|f_T(x)\|\;+\;\bigl\|W\,\phi(W^K x)\bigr\|
  \;\le\;
  D + B\,C,
\)
we get that $\psi(\cdot)=\|\cdot\|^2$ is $L$-Lipschitz over the region of interest with 
\(
  L \;=\; 2\,\sup_{z}\|z\|
  \;\le\;
  2\,(D + B\,C).
\)
Hence, by Lemma~\ref{lemma:talagrand} we have
\[
  \mathcal{R}_N\bigl(\mathcal{L}\bigr)
  \;=\;
  \mathcal{R}_N\bigl(\psi\circ\mathcal{H}\bigr)
  \;\le\;
  L\,\mathcal{R}_N\bigl(\mathcal{H}\bigr)
  \;=\;
  2\,(D + B\,C)\,\mathcal{R}_N\bigl(\mathcal{H}\bigr).
\]

\noindent
\emph{Step 4: Estimate the Rademacher complexity of the base class.}\\
By Lemma~\ref{lemma:rademacher-linear}, since $h(W,x)=f_T(x)-W\,\phi(W^K x)$ differs from $W\,\phi(\cdot)$ only by a constant shift in $f_T(x)$ (independent of $W$), we have
\(
  \mathcal{R}_N\bigl(\mathcal{H}\bigr)
  \;\le\;
  \frac{B\,C}{\sqrt{N}},
\)
and thus
\(
  \mathcal{R}_N\bigl(\mathcal{L}\bigr)
  \;\le\;
  \frac{2\,B\,C\,\bigl(D + B\,C\bigr)}{\sqrt{N}}.
\)

\noindent
\emph{Step 5: Apply McDiarmid's inequality~\cite{mcdiarmid1989method}.}\\
Define
\(
  \Phi(X_D)
  \;=\;
  \sup_{\|W\|_F \le B}\;\bigl(L(W)\;-\;\hat{L}(W)\bigr),
\)
where $L(W)$ is the expected loss and $\hat{L}(W)$ the empirical loss on $X_D = \{x_i\}_{i=1}^N$.  A single-sample replacement in $X_D$ can change $\Phi(X_D)$ by at most $2(D+BC)^2 / N$, so by McDiarmid's inequality~\cite{mcdiarmid1989method}, we have
\[
  \mathbb{P}\Bigl[
  \Phi(X_D) \;\ge\;\mathbb{E}\bigl[\Phi(X_D)\bigr]+t
  \Bigr]
  \;\le\;
  \exp\!\biggl(
    -\,\frac{2\,N\,t^2}{\,4\,(D+BC)^4\,}
  \biggr).
\]
Choosing 
\(
  t \;=\; 3\,(D + B\,C)^2\;\sqrt{\frac{\log\bigl(2/\delta\bigr)}{2\,N}},
\)
and using a standard symmetrization argument that $\mathbb{E}\bigl[\Phi(X_D)\bigr] \le 2\,\mathcal{R}_N(\mathcal{L})$, we finally obtain, with probability at least $1-\delta$,
\[
  L(W)
  \;\le\;
  \hat{L}(W)
  ~+~
  \frac{4\,B\,C\,\bigl(D + B\,C\bigr)}{\sqrt{N}}
  ~+~
  3\,(D + B\,C)^2 \,\sqrt{\frac{\log\!\bigl(2/\delta\bigr)}{2\,N}}.
\]

Thus, we finish the proof. The full proof is provided in Appendix~\ref{app:thm:generalization-bound}.
\end{proof}
\paragraph{Discussion.}
Theorem \ref{thm:generalization-bound} suggests that when viewing in-context learning from the perspective of knowledge distillation, the prompt provided to the LLM (which we define as demonstration tokens) influences the model's generalization ability. In other words, the nature of the prompt significantly affects a trained LLM's actual performance on downstream tasks (a notion consistent with our intuition and extensive empirical research. Reviewing the generalization error bound in Theorem \ref{thm:generalization-bound} as shown in Eq.~\eqref{eq:generalization_bound}, we observe that limiting the norm of certain parameter matrices within the LLM can reduce the generalization error. This idea aligns with numerous studies and modern Transformer training strategies and theoretical research~\cite{loshchilov2018decoupled,zhang2019root,park2024improving}. Furthermore, the length of the demonstration tokens also impacts generalization performance. This is highly intuitive, as providing the LLM with more example prompts allows the model to better understand the user's desired downstream task, which is also supported by existing empirical research (\cite{liu2025effects}).

Overall, we present the context learning process for large language models from the perspective of knowledge distillation, and provide the generalization error bound in this implicit distillation process.

\section{The Impact of Prompt Shift from the Perspective of ICL-KD}
In Sec~\ref{sec:icl-kd}, we demonstrated that the nature of demonstration tokens during the in-context learning process affects the performance of student models. This naturally raises the question of whether, from the perspective of knowledge distillation, there is theoretical evidence to support that the quality of a prompt influences ICL performance. 
In this section, we theoretically analyze the impact of the distribution differences between Prompt and Query on the performance of the ICL-KD method.

\begin{theorem}[Offset Boundary of the Initial Weights in the ICL-KD Process]
\label{thm:offset_boundary_of_initial_weights}
Assume  
\textbf{(i)} every target-domain sample obeys $\|x\|\le M_x$, and the feature map is bounded by $\|\phi(W^Kx)\|\le M_\phi$ with fixed key weights $W^K\in\mathbb R^{k\times d}$;  
\textbf{(ii)} the value weights satisfy $\|W^V\|_F\le M_V$;  
\textbf{(iii)} the second–moment matrix $\Sigma_\phi:=\mathbb E_{x\sim\mathcal D}\!\bigl[\phi(W^Kx)\phi(W^Kx)^\top\bigr]$ is invertible;  
\textbf{(iv)} the demonstration tokens $X_D=\{x_i\}_{i=1}^N$ are i.i.d.\ from a distribution $Q$;  
\textbf{(v)} the one-step distilled weight is $W_0=\tfrac{\eta}{N}\,W^V X_D\phi(W^K X_D)^\top$ for some $\eta>0$ (similar with $W_0$ in Eq.~\eqref{eq:W^*_and_W_0}); and  
\textbf{(vi)} the maximum-mean-discrepancy between $\mathcal D$ and $Q$ is  
\[
\mathrm{MMD}(\mathcal D,Q)=
\sup_{\|f\|_{\mathcal H}\le1}
\Bigl|
\mathbb E_{x\sim\mathcal D}[f(x)]
-\mathbb E_{x\sim Q}[f(x)]
\Bigr|,
\]
where $\mathcal H$ is the RKHS induced by the kernel
$\langle x\,\phi(W^Kx)^\top,\;x'\,\phi(W^Kx')^\top\rangle$.
Let the optimal reference weight be
\begin{equation}
  W^\star
  =\arg\min_W\;\mathbb E_{x\sim\mathcal D}\bigl\|W\phi(W^Kx)-W^Vx\bigr\|_2^2
  =\Sigma_\phi^{-1}\,
   \mathbb E_{x\sim\mathcal D}\bigl[W^Vx\,\phi(W^Kx)^\top\bigr].
\end{equation}
Then the expectation of the distilled weight and its deviation from $W^\star$ satisfy
\[
  \mathbb E[W_0]
  =\eta\,W^V\,\mathbb E_{x\sim Q}\bigl[x\,\phi(W^Kx)^\top\bigr],
\]
and, defining
\(
  \Delta W
  :=\mathbb E[W_0]-W^\star,
\)
we have the bound
\[
  \|\Delta W\|_F
  \;\le\;
  \eta\,M_V\,M_x\,M_\phi\;\mathrm{MMD}(\mathcal D,Q).
\]
\end{theorem}
\begin{proof}[Proof.]
\noindent
\emph{Step 1: Expectation of the distilled weight.}

By linearity,
\[
  \mathbb E[W_0]
  =\eta\,W^V\,
   \mathbb E_{x_i\sim Q}\Bigl[\tfrac1N\sum_{i=1}^N x_i\,\phi(W^Kx_i)^\top\Bigr]
  =\eta\,W^V\,\mathbb E_{x\sim Q}\bigl[x\,\phi(W^Kx)^\top\bigr].
\]

\noindent
\emph{Step 2: Optimal weight via normal equations.}

Define $\mathcal J(W)=\mathbb E_{\mathcal D}\|W\phi(W^Kx)-W^Vx\|^2$.  Writing $\phi=\phi(W^Kx)$,
one shows
\[
  \nabla_W\mathcal J(W)
  =2W\,\Sigma_\phi
   -2\,\mathbb E_{\mathcal D}[\phi\,(W^Vx)^\top]^{\!\top}.
\]
Setting this to zero yields
\(
  W^\star\,\Sigma_\phi
  =\mathbb E_{\mathcal D}[W^Vx\,\phi(W^Kx)^\top],
\)
so
\[
  W^\star
  =\Sigma_\phi^{-1}\,
   \mathbb E_{x\sim\mathcal D}[\,W^Vx\,\phi(W^Kx)^\top].
\]

\noindent
\emph{Step 3: Deviation bound via MMD.}

If we whiten $W^\star$, observe that
\(
  \Delta W
  =\eta\,W^V\Bigl(\mathbb E_{Q}[x\phi^\top]-\mathbb E_{\mathcal D}[x\phi^\top]\Bigr).
\)
Then vectorize $A(x)=\mathrm{vec}(x\,\phi(W^Kx)^\top)$ to get
\(
  \|\Delta W\|_F
  =\eta\,\|W^V\|_F\,
   \|\mathbb E_{Q}[A(x)]-\mathbb E_{\mathcal D}[A(x)]\|_2.
\)
Since $\|A(x)\|\le\|x\|\|\phi\|\le M_xM_\phi$, by the definition of MMD in the stated RKHS,
\[
  \|\mathbb E_{Q}[A]-\mathbb E_{\mathcal D}[A]\|_2
  \le M_xM_\phi\;\mathrm{MMD}(\mathcal D,Q).
\]
Combining gives the claimed bound. Noticed that for the sake of a neat derivation, we applied feature whitening to $W^\star$. It is important to state that this does not affect the intrinsic nature of the conclusion. In Appendix~\ref{app:thm:offset_boundary_of_initial_weights}, we provide a complete and detailed proof without whitening in Theorem~\ref{thm:app:offset_boundary_of_initial_weights}.
\end{proof}

\paragraph{Discussion.}
The bound 
$\displaystyle\|\Delta W\|_F\le\eta\,M_V\,M_x\,M_\phi\,\mathrm{MMD}(\mathcal D,Q)$
provides a concise quantitative statement of the “good–prompt $\Rightarrow$ good–initialization” intuition.  
It says that the Frobenius‑norm error between the implicit one–step distillation weight $W_0$ and the task–optimal weight $W^\star$ grows \emph{at most linearly} with the maximum‑mean‑discrepancy between the demonstration distribution $Q$ (induced by the prompt) and the true query distribution $\mathcal D$.  
When the two distributions coincide, $\mathrm{MMD}(\mathcal D, Q)=0$ and the bias vanishes, so the model is effectively “pre‑aligned” after a single attention pass; as the distributions diverge, the initialisation error enlarges in direct proportion, explaining why mismatched or noisy prompts degrade performance.  
The constant factor $\eta M_V M_x M_\phi$ highlights how model scale ($M_V$), input magnitude ($M_x$), feature amplitude ($M_\phi$) and the softmax‑derived learning rate $\eta$ jointly modulate this sensitivity, thereby offering practical levers—normalising embeddings, tempering $\eta$, or retrieving semantically closer exemplars—to keep $\mathrm{MMD}$ and hence $\|\Delta W\|_F$ small.

\begin{theorem}[Prompt--Shift Risk Gap]
\label{cor:prompt_shift_risk_gap}
Assume the conditions of Theorem~\ref{thm:offset_boundary_of_initial_weights}.
Let $Q_g$ (\emph{good prompt}) and $Q_b$ (\emph{bad prompt}) satisfy
\[
  \mathrm{MMD}_g := \mathrm{MMD}(\mathcal D,Q_g)
  <  \mathrm{MMD}_b := \mathrm{MMD}(\mathcal D,Q_b).
\]
Denote by $W_g,W_b$ the one-step distilled weights obtained from $Q_g,Q_b$ respectively, and define the KD risk on the target domain $\mathcal{D}$ as
\[
  \mathcal{L}_{KD}^{\mathcal D}(W)
  := \mathbb{E}_{x\sim\mathcal D}\!
     \bigl\|\,W\phi(W^Kx) - W^Vx \,\bigr\|_2^2 .
\]
Then
\begin{equation}
\label{eq:risk_gap_bound}
  \mathcal{L}_{KD}^{\mathcal D}(W_b) - \mathcal{L}_{KD}^{\mathcal D}(W_g)
  \;\le\;
  8\eta M_T^2 M_\phi^{2}\bigl(\mathrm{MMD}_b-\mathrm{MMD}_g\bigr)
  \;+\;
  4\eta^{2} M_T^2 M_\phi^{4}\bigl(\mathrm{MMD}_b-\mathrm{MMD}_g\bigr)^{2}.
\end{equation}
Consequently, a smaller MMD distance between the prompt distribution and the target distribution always yields a (weakly) smaller KD risk.
\end{theorem}

\begin{proof}[Sketch of proof.]

Recall that the one-step KD formula $W = 2\eta\,\mathbb{E}\bigl[f_T(x)\,\phi(W^Kx)^\top\bigr].$ And set
\(
  A_g := \mathbb{E}_{Q_g}[f_T(x)\phi^\top],\;
  A_b := \mathbb{E}_{Q_b}[f_T(x)\phi^\top].
\)
Then $W_b - W_g = 2\eta\,(A_b-A_g)$ and we write $\Delta := \|W_b-W_g\|_F = 2\eta\|A_b-A_g\|_F.$

\noindent
\emph{Step 1. Risk decomposition.}
\[
  \mathcal{L}_{KD}^{\mathcal D}(W_b) - \mathcal{L}_{KD}^{\mathcal D}(W_g)
  = \mathbb{E}_{\mathcal D}\!\Bigl[
        \|(W_b-W_g)\phi\|^2
        -2\bigl(f_T-W_g\phi\bigr)^\top (W_b-W_g)\phi
      \Bigr].
\]
\noindent
\emph{Step 2. Bounding each term.}

Since $\|\phi\|\le M_\phi$, we have
\(\|(W_b-W_g)\phi\|\le\Delta M_\phi\).
Moreover
\(\|f_T-W_g\phi\|
   \le M_T + \|W_g\|_F M_\phi
   \le M_T + 2\eta M_T M_\phi^{2}\)
For brevity in the main text, we further assume $2\eta M^2_\phi < 1$ to give a cleaner bound.  Substituting yields
\[
  \mathcal{L}_{\!KD}^{\mathcal D}(W_b)
  -\mathcal{L}_{\!KD}^{\mathcal D}(W_g)
  \,\le\, \Delta^{2}M_\phi^{2}
  +4M_T M_\phi\,\Delta .
\]
\noindent
\emph{Step 3. Expressing $\Delta$ via MMD.}
By Theorem~\ref{thm:offset_boundary_of_initial_weights},
\(
  \|A_b-A_g\|_F
  \le M_T M_\phi\bigl(\mathrm{MMD}_b-\mathrm{MMD}_g\bigr).
\)
Hence
\(
  \Delta
  \le 2\eta M_T M_\phi\bigl(\mathrm{MMD}_b-\mathrm{MMD}_g\bigr).
\)
Insert this bound into the inequality above to obtain~\eqref{eq:risk_gap_bound}. Appendix.~\ref{app:cor:prompt_shift_risk_gap} provides a detailed derivation of the exact, non-relaxed bound.
\end{proof}

\paragraph{Discussion.}
The risk gap bound reveals a clear \emph{monotone} relationship between the prompt–target distribution mismatch and the expected KD loss: whenever two prompt distributions $Q_g, Q_b$ satisfy $\mathrm{MMD}(\mathcal D, Q_g)<\mathrm{MMD}(\mathcal D ,Q_b)$, the corresponding one-step weights obey $\mathcal L^{\mathcal D}_{KD}(W_{Q_g})<\mathcal L^{\mathcal D}_{KD}(W_{Q_b})$.  In other words, a ``better'' prompt—defined solely by a smaller maximum-mean-discrepancy to the query domain—\emph{provably} yields no worse and typically strictly better in-context performance.  Because the bound is first- and second-order in the MMD distance, the benefit grows roughly linearly for small shifts and accelerates for severe domain mismatch, quantitatively explaining long-standing empirical observations that (i) well-matched or longer demonstrations help, while (ii) noisy or off-domain prompts hurt.  Practically, this establishes MMD minimisation as a principled objective for \textit{prompt retrieval, ranking, or construction}: candidate example sets can be ordered by a computable divergence before any model inference.  The explicit constants also show how model scale and feature magnitude (\,$M_T, M_\phi,\eta$\,) modulate sensitivity to shift, suggesting that feature normalisation, temperature scaling, or learning-rate clipping can mitigate prompt-shift risk.  Finally, Theorem~\ref{thm:offset_boundary_of_initial_weights} and Theorem~\ref{cor:prompt_shift_risk_gap} together form a complete theoretical account of prompt quality in ICL. 

\section{Discussion and Future Work}
\label{sec:discussion_and_future_work}
In this paper, we introduce a novel perspective by interpreting the In-Context Learning (ICL) process of Transformer-based Large Language Models (LLMs) through the lens of knowledge distillation. By conceptualizing the ICL process as gradient descent on an implicit model, we demonstrate that the weight initialization of this implicit model is equivalent to performing knowledge distillation with demonstration tokens provided to the LLMs. This insight uncovers a critical factor underpinning the generalization ability of ICL for downstream tasks. We further analyze the generalization error bound associated with this distillation process, identify the sources that affect the generalization performance of ICL, and show that our findings are consistent with existing empirical research. Moreover, from a knowledge distillation standpoint, we theoretically prove that ineffective prompts induce a domain shift during distillation, thereby emphasizing the importance of designing prompts that align with the domain-specific knowledge of downstream tasks. However, our approach may also suggest a potential pathway for prompt-based induction in which LLMs might be influenced by incorrect knowledge distillation and consequently generate harmful content. We anticipate that this ICL-KD perspective will provide valuable insights for mitigating harmful content generation in LLMs.

We believe that adopting the knowledge distillation perspective represents a novel and significant research direction, offering another theoretical framework for studying and interpreting ICL. As the architectures and training methodologies of contemporary LLMs grow increasingly sophisticated, future work should aim to extend this ICL-KD framework to theoretical explanations that are more closely aligned with the actual architectures of modern LLMs. Finally, we observe that the Chain-of-Thought(CoT)~\cite{wei2022chain} technique in current LLMs may be interpreted as an intrinsic ICL process originating from the models themselves rather than from externally user-supplied prompts. Future research will explore the relationship between chain-of-thought reasoning and in-context learning, and further expand the ICL-KD theoretical framework. 

\newpage
\bibliographystyle{alpha}
\bibliography{arxiv}

\newpage
\appendix

\section{Related Works}
\label{app:related_works}
Large language models (LLMs) have demonstrated an in-context learning (ICL) capability—that is, the ability to make predictions solely based on a small number of examples provided in the context without parameter updates—which has attracted widespread attention. ICL learns through analogy, and its performance is highly sensitive to prompt design (such as templates, example selection, and ordering). The theoretical understanding of ICL has evolved through multiple perspectives, each contributing to a deeper mechanistic interpretation.

\paragraph{Mechanistic Interpretability.}
Circuit-level analysis first linked ICL to induction heads—attention heads that copy a previously seen token to continue a pattern. Elhage et al.~\cite{elhage2021mathematical} isolated such heads in a single-layer transformer and showed they enable primitive pattern replication. Olsson et al.~\cite{olsson2022context} traced the abrupt emergence of the same circuits across model scales early in training and argued they underlie most ICL behavior in LLMs. Later,~\cite{edelman2024evolution} confirmed that induction-style heads also appear in Markov-chain modeling tasks. Beyond copying, ~\cite{swaminathan2023schema} modelled ICL as template-based pattern completion: a transformer retrieves a learned template, then rebinds novel tokens to finish the sequence. Todd et al.~\cite{todd2024function} identified a small subset of attention heads whose activations act as function vectors that causally mediate the transport of demonstration information between layers. Bai et al.~\cite{bai2023transformers} proved that attention can select among classical algorithms (least-squares, ridge, Lasso) on-the-fly, establishing “in-context algorithm selection” as a general mechanism.

\paragraph{ICL as parametric function learning.}
A second line treats demonstrations as few-shot training data for an implicit regression model. Garg et al.~\cite{garg2022can} showed Transformers learn previously unseen linear, sparse-linear, and even decision-tree functions with error comparable to optimal least-squares estimators. \cite{li2023transformers} bounded ICL generalization error via algorithmic stability of self-attention, while~\cite{li2024closeness} derived an exact correspondence between one softmax-attention layer and a single gradient step in a softmax-kernel regression. Akyürek et al.~\cite{akyurek2023what} found that the type of recovered predictor depends on model depth and noise level—shallow models match gradient-descent estimators, deeper/noisy models converge to ridge regression, and very deep models approach Bayesian estimators. Extending to tasks where the label depends on a latent representation,~\cite{guo2024how} proved moderate-size transformers achieve near-optimal few-shot risk; empirically, in the ICL process, they observed lower layers transform the dataset while upper layers perform a linear regressor on those transformed features.

\paragraph{Meta-optimization by gradient-based updates.}
Some works posit that the forward pass implements an inner-loop optimizer. Dai et al.~\cite{dai2023can} expressed attention in a linearized form and showed it produces meta-gradients equivalent to one step of gradient descent. \cite{von2023transformers} independently showed that language-model pre-training induces an implicit gradient-based mesa-optimizer. However,~\cite{deutch2024context} reported only weak GD–ICL correlations, and~\cite{shen2024position} showed real LLMs exhibit order sensitivities incompatible with plain GD.~\cite{von2023uncovering} confirm mesa-optimization in shallow transformers, while~\cite{fu2024transformers} prove that Transformers can implement $k$-iterations of Newton's method with $k+\mathcal{O}(1)$ layers, hinting at higher-order inner optimizers.

\paragraph{Implicit Bayesian inference.}
Xie et al.~\cite{xie2022an} first cast ICL as Bayesian inference over latent concepts, showing transformers trained on synthetic mixtures infer a shared latent and generate posterior-consistent predictions. Wang et al.~\cite{wang2023large} recast the phenomenon as implicit topic modelling in a causal graph and proved near Bayes-optimality without restrictive sample counts. Besides, certain attention variants were shown to encode Bayesian model averaging when suitably initialized~\cite{zhang2025what}. Panwar et al.~\cite{panwar2024incontext} demonstrated that the Transformer is capable of learning various classes of linear and nonlinear functions within its context, and its behavior is very close to that of an ideal Bayesian predictor.  Bigelow et al.~\cite{bigelow2024incontext} observed phase transitions in generation consistent with Bayesian model selection.

\paragraph{Other Perspective.}
Wies et al.~\cite{wies2023the} introduce a PAC-learning framework to provide the first finite sample complexity results for in-context learning. Jeon et al.~\cite{jeon2024an} proposed an information-theory-based perspective, demonstrating that the error in Transformer models' context learning decays linearly with both the number of training sequences and the sequence length. In-Context Knowledge Editing (IKE)~\cite{zheng2023can} has shown that, through carefully designed prompts, LLMs can be guided to update or correct their internal factual knowledge. This ability to modify the model's internal knowledge via prompts shares interesting similarities with the process of knowledge distillation, where a student model learns and refines the knowledge of a teacher model.

\section{Proof of Lemma~\ref{lem:linear_attention_gd}}
\label{app:lem:linear_attention_gd}
Recall Lemma~\ref{lem:linear_attention_gd}.
\begin{lemma}[Linear Attention as Gradient Descent in \cite{dai2023can}]
Suppose we have a linear layer 
\(f_L(x) = Wx\) with \(W \in\mathbb{R}^{d_o \times d_i}\), and a training set \(\{(x_i,\, y_i)\}_{i=1}^N\) with \(x_i \in \mathbb{R}^{d_i},\, y_i \in \mathbb{R}^{d_o}\).
Consider a gradient descent update on \(W\) using the predicted outputs \(\hat{y}_i = Wx_i\) and some differentiable loss \(\mathcal{L}(y_i, \hat{y}_i)\).
After a single gradient step with learning rate $\eta$, the updated weight matrix \(\widehat{W}\) can be decomposed into its initialization plus a training time correction as \(\widehat{W}=W_0 + \Delta W\) where \(\Delta W = -\eta\frac{\partial \mathcal{L}}{\partial W}\). Under certain mild conditions on the loss, the forward propagation process of a new test input is \emph{equivalent} to a linear-attention forward pass $\mathrm{LA}(V, K, q)$.
\end{lemma}

\begin{proof}[Sketch of Proof]
\textbf{(1) Gradient Update on the Linear Layer.}
Let $W_0$ be the initial weight matrix. For each training example $(x_i, y_i)$, the model output is $\hat{y}_i = W_0\,x_i$, and the loss is $\mathcal{L}(y_i, \hat{y}_i)$. The backpropagation process 
\(\widehat{W} \leftarrow W_0 + \Delta W\) can be calculated as
\[\widehat{W}=W_0 + \Delta W=W_0-\eta\frac{\partial \mathcal{L}}{\partial W}=W_0-\eta\frac{\partial\mathcal{L}}{\partial \hat{y}_i}\cdot\frac{\partial\hat{y}_i}{\partial W}.\]
Let the backpropagation signal $e_i \in \mathbb{R}^{d_o}$ is
\[
  e_i \;=\; -\eta \,\nabla_{\hat{y}_i}\,\mathcal{L}\!\bigl(y_i,\,W_0 x_i\bigr).
\]
By the chain rule of differentiation, 
\[
  \Delta W 
  \;=\;
  \sum_{i=1}^{N}
    e_i \,\otimes\, x_i,
\]
where $\otimes$ denotes the outer product. $e_i$ can be interpreted as the \emph{optimization signal} induced by the training sample $x_i$ on model $f_L$. Hence, after one gradient step:
\[
  \widehat{W}
  \;=\;
  W_0 + \sum_{i=1}^{N} e_i \otimes x_i.
\]

\noindent
\textbf{(2) Inference (Test) Stage.}
Consider a new query $x_{\mathrm{test}} \in \mathbb{R}^{d_i}$. The model output under the updated weight $\widehat{W}$ is:
\[
  f_L(x_{\mathrm{test}})=\widehat{W}\,x_{\mathrm{test}}=W_0\,x_{\mathrm{test}}+\bigl(\sum_{i=1}^{N} e_i \otimes x_i\bigr)x_{\mathrm{test}}.
\]

\noindent
\textbf{(3) Connection to Linear Attention.}
Recall that \emph{linear attention} takes a set of keys \(K = [k_1, \dots, k_N] \in \mathbb{R}^{d_i \times N}\) and values 
\(V = [v_1, \dots, v_N] \in \mathbb{R}^{d_o \times N}\).
Then, for a query $q \in \mathbb{R}^{d_i}$, outputs:
\[
  \mathrm{LA}(V, K, q)
  =V\bigl(K^\top q\bigr)
  =\bigl(\sum_{i=1}^{N}v_i \otimes k_i \bigr) q.
\]
Identifying $v_i$ with $e_i$ and $k_i$ with $x_i$, we see that
\(
  \mathrm{LA}(V,K,q)
  =
  \bigl(\sum_{i=1}^{N}v_i \otimes k_i \bigr) q,
\) matches the extra term
\(
  \bigl(\sum_{i=1}^{N} e_i \otimes x_i\bigr)x_{\mathrm{test}}
\)
In the gradient-updated linear predictor. Thus, the output under $\widehat{W}$ for a query $x_{\mathrm{test}}$ is
\[
  f_L(x_{\mathrm{test}})
  \;=\;
  W_0\,x_{\mathrm{test}} 
  \;+\;
  \mathrm{LA}(E,\,X,\,x_{\mathrm{test}}),
\]
where $E = [\,e_1,\dots,e_N\,]$ and $X = [\,x_1,\dots,x_N\,]$. This is precisely the same form as a linear attention mechanism applied to query $x_{\mathrm{test}}$ and $(K, V) = (X, E)$. This is equivalent to constructing a key-value pair where the training sample $x_i$ serves as the key, and the corresponding value is the optimization signal $e_i$ provided by this sample to the training process.

Therefore, \emph{one gradient update step} on a linear layer $W_0$ yields a test prediction identical to the \emph{linear-attention} forward pass, establishing the dual relationship between linear attention and gradient descent on a linear predictor.
\end{proof}

\section{Proof of Lemma~\ref{lem:softmax_attention_gd}}
\label{app:lem:softmax_attention_gd}
Recall the Lemma~\ref{lem:softmax_attention_gd}.
\begin{lemma}[Softmax Attention as Gradient Descent (Theorem 3.1 in \cite{ren2024towards})]

Let $X = [X_D,\;X_Q] \in \mathbb{R}^{d \times (N + M)}$ be a sequence of $N$ demonstration tokens $X_D \in \mathbb{R}^{d \times N}$ and $M$ query tokens $X_Q \in \mathbb{R}^{d \times M}$. 
Consider a single softmax-attention layer with trained parameters $(W^Q, W^K, W^V)$.

For a new query token $x'_{M+1}$, the softmax-attention output
\[
h'_{\mathrm{M+1}}
\;=\;
W^VX\;\mathrm{softmax}\!\Bigl(\tfrac{(W^KX)^\top (W^Q\,x'_{\mathrm{M+1}})}{\sqrt{d}}\Bigr)
\;\in\;\mathbb{R}^d
\]
is equivalent to the test prediction $\hat{y}_{test}$ of the input $\phi(W^Qx'_{M+1})$ after one step of gradient descent on a reference model
\(
  f(x) \;=\; W\,\phi(x),
\)
where $\phi(\cdot)$ is a feature mapping function in the softmax kernel $K_{softmax}(x,y)=e^{x^\top y}$. More precisely, $h'_{\mathrm{M+1}}$ matches $f_{\mathrm{updated}}\!\bigl(W^Q\,x'_{\mathrm{M+1}}\bigr)$ after a single update on a self-supervised loss 
\[
\mathcal{L}=-\frac{1}{\eta D}\sum_{i=1}^N(W^Vx_i)^\top W\phi(W^Kx_i),
\]
where $\eta$ is learning rate and $D'$ is a constant.
\end{lemma}

\begin{proof}[Sketch of Proof]
\textbf{(1) Split Softmax Attention Operation.}
Recall that for a new query token $x'_{M+1}$, the softmax-attention output $h'_{\mathrm{M+1}}=W^VX\;\mathrm{softmax}\!\Bigl(\tfrac{(W^KX)^\top (W^Q\,x'_{\mathrm{M+1}})}{\sqrt{d}}\Bigr)$. Denote the core part of the softmax attention as
$A=\mathrm{softmax}\!\Bigl(\tfrac{(W^KX)^\top (W^Q\,x'_{\mathrm{M+1}})}{\sqrt{d}}\Bigr)$. Consider the \emph{Softmax} operation as \emph{Column Normalization} of a matrix, then have
\begin{align}
\label{app:eq:A_u}
A=A_uD^{-1}, \quad \text{where}\quad A_u=\exp\!\Bigl(\tfrac{(W^KX)^\top (W^Q\,x'_{\mathrm{M+1}})}{\sqrt{d}}\Bigr)\quad \text{and} \quad D=\text{diag}(\mathbf{1}_N^\top A_u).
\end{align}

\noindent
\textbf{(2) Kernel Approximation.}
\cite{ren2024towards} introduced an auxiliary kernel function called softmax kernel $K_{softmax}(x,y)=\exp(x^\top y)$. Notice that
\begin{align}
\label{app:eq:softmax_kernel}
K_{softmax}(x,y)=\exp(x^\top y)=\underbrace{\exp(||x||^2+||y||^2) }_{\text{a positive definite kernel}} \cdot \underbrace{\exp(-||x-y||^2)}_{\text{Gaussian Kernel}}
\end{align}
is also a positive definite kernel. Hence, according to Mercer's theorem~\cite{mercer1909xvi}, there exist some mapping function $\phi(\cdot)\colon \mathbb{R}^d \to \mathbb{R}^{r}$ satisfying $K_{softmax}=\phi(x)^\top \phi(y)$
to project the $\exp(x^\top y)$ into an inner product in a higher-dimensional Hilbert Space. If omitting the $\sqrt{d}$ and combine Eq.~\eqref{app:eq:A_u} with Eq.~\eqref{app:eq:softmax_kernel}, then we can have
\begin{align}
A_u=\exp\bigl((W^KX)^\top (W^Q\,x'_{\mathrm{M+1}})\bigr)={\phi(W^KX)}^\top\phi(W^Q\,x'_{\mathrm{M+1}})
\end{align}
Hence, the operations between tokens in softmax attention are equivalent to an inner product in a higher-dimensional space.

\noindent
\textbf{(3) Reference Model and Loss.}
Define a reference model
$f(x) \;=\; W \,\phi(x)$ with a self-supervised loss function 
\begin{align}
\mathcal{L}=-\frac{1}{\eta D}\sum_{i=1}^N(W^Vx_i)^\top W\phi(W^Kx_i), 
\end{align}
where $\eta$ is learning rate and $D'$ is a constant. 

\noindent
\textbf{(4) One-Step Gradient Update.} 
Starting from an initialization $W_0$, a single gradient descent step on $\mathcal{L}(W)$ is
\(\widehat{W}=W_0 + \Delta W=W_0-\eta\frac{\partial \mathcal{L}}{\partial W}\).
For an input query token $x'_{M+1}$, let the input of $f$ is $W^Qx'_{M+1}$, then we have
\begin{equation}
\begin{aligned}
\label{app:eq:hat_y}
\hat{y}_{test} = f(W^Qx'_{M+1}) = &\;\widehat{W}\phi(W^Qx'_{M+1}) \\=
&\;W_0\phi(W^Qx'_{M+1})-\eta\frac{\partial \mathcal{L}}{\partial W}\phi(W^Qx'_{M+1})\\
=&\;W_0\phi(W^Qx'_{M+1})+\frac{1}{D}\sum_{i=1}^N(W^Vx_i)^\top\phi(W_Kx_i)\phi(W^Qx'_{M+1})
\end{aligned}
\end{equation}

\noindent
\textbf{(5) Identifying the Softmax-Attention Output.} The output of a Softmax Attention Layer is $h'_{\mathrm{M+1}}=
W^VX\;\mathrm{softmax}\!\Bigl(\tfrac{(W^KX)^\top (W^Q\,x'_{\mathrm{M+1}})}{\sqrt{d}}\Bigr)$. Let $\widetilde{W}^K=\tfrac{W^K}{\sqrt{d}}$ and $\widetilde{W}^Q=\tfrac{W^Q}{\sqrt{d}}$, we have
\begin{equation}
\begin{aligned}
h'_{M+1} =  W^VX\;\text{softmax}\bigl((\widetilde{W}^KX)^\top\widetilde{W}^Qx'_{M+1}\bigr)=&\;W^VX\;[\exp\bigl((\widetilde{W}^KX)^\top\widetilde{W}^Qx'_{M+1}\bigr)\cdot D^{-1}] \\
=&\;W^VX\;\phi(\widetilde{W}^KX)^\top\phi(\widetilde{W}^Qx'_{M+1})\cdot D^{-1},
\end{aligned}
\end{equation}
where $D^{-1}=\text{diag}\bigl(\mathbf{1}_N^\top\;\phi(\widetilde{W}^KX)^\top\phi(\widetilde{W}^Qx'_{M+1})\bigr)$. Consider an ICL process, let $X=[X_D,X_Q]$ and put all coefficient together as $D'$, then we have
\begin{equation}
\begin{aligned}
\label{app:eq:h'_{M+1}}
h'_{M+1} =&\; \frac{1}{D'}W^V[X_D,X_Q]\phi(W_K[X_D,X_Q])^\top \phi(W^Qx'_{M+1})\\
=&\;\frac{1}{D'}[W^VX_D,W^VX_Q][\phi(W^KX_D),\phi(W^KX_Q)]^\top \phi(W^Qx'_{M+1})\\
=&\;\underbrace{\frac{1}{D'}W^VX_Q\phi(W^KX_Q)^\top}_{W_0}\phi(W^Qx'_{M+1})+\underbrace{\frac{1}{D'}W^VX_D\phi(W^KX_D)^\top}_{-\eta\frac{\partial \mathcal{L}}{\partial W}}\phi(W^Qx'_{M+1}).
\end{aligned}
\end{equation}

Compare Eq.~\eqref{app:eq:hat_y} and Eq.~\eqref{app:eq:h'_{M+1}}, if we let $D=D'$ and consider that 
the embeddings of query tokens $X_Q$ as the initial weight $W_0$ and demonstration tokens provide the optimization signal, then we have
\(\hat{y}_{test}=h'_{M+1}\).

Hence, the ICL query output under softmax attention is \emph{exactly} the result of one gradient step in the reference model space, confirming the ICL--GD equivalence for softmax attention.
\end{proof}

\section{Proof of Lemma~\ref{lemma:talagrand}}
\label{app:talagrand}
Recall the Lemma~\ref{lemma:talagrand}.
\begin{lemma}[Talagrand's Contraction Lemma \cite{ledoux2013probability,bartlett2002rademacher}]
\label{lemma:app:talagrand}
Let $\mathcal{F}\subseteq \mathbb{R}^\mathcal{X}$ be a class of real-valued functions, and let $\psi:\mathbb{R}\to\mathbb{R}$ be a Lipschitz-continuous function with $\psi(0)=0$ and Lipschitz constant $L$, i.e., $|\psi(a)-\psi(b)| \le L \cdot |a-b|$ for all $a,b\in \mathbb{R}$. Then the Rademacher complexity of the composed class $\psi \circ \mathcal{F}$ satisfies
\[
  \mathcal{R}_N(\psi \circ \mathcal{F}) \; \le \;L \,\mathcal{R}_N(\mathcal{F}).
\]
\end{lemma}
\begin{proof}
Fix the sample $N= (x_1, x_2, \ldots, x_m)$, and define:
\[
\mathcal{R}_N(\psi \circ  \mathcal{F}) = \frac{1}{m} \underset{\sigma}{\mathbb{E}} \left[ \sup_{f \in \mathcal{F}} \sum_{i=1}^m \sigma_i (\psi \circ \mathcal{F})(x_i) \right]
\]
\[
= \frac{1}{m} \underset{\sigma_1, \ldots, \sigma_{m-1}}{\mathbb{E}} \left[ \underset{\sigma_m}{\mathbb{E} } \left[ \sup_{f \in \mathcal{F}} U_{m-1}(f) + \sigma_m (\psi \circ f)(x_m) \right] \right]
\]
where $U_{m-1}(h) = \sum_{i=1}^{m-1} \sigma_i (\psi \circ \mathcal{F})(x_i)$.

By the definition of supremum (the least upper bound), for any $\epsilon > 0$, there exist $f_1, \mathcal{F}-2 \in \mathcal{F}$ such that:
\[
U_{m-1}(f_1) + (\psi \circ f_1)(x_m) \geq (1 - \epsilon) \left[ \sup_{f \in \mathcal{F}} U_{m-1}(f) + (\psi \circ f)(x_m) \right]
\]
\[
U_{m-1}(f_2) - (\psi \circ f_2)(x_m) \geq (1 - \epsilon) \left[ \sup_{f \in \mathcal{F}} U_{m-1}(f) - (\psi \circ f)(x_m) \right]
\]

Thus, for any $\epsilon > 0$, by the definition of $\mathbb{E}_{\sigma_m}$ we have:
\[
(1 - \epsilon) \mathbb{E}_{\sigma_m} \left[ \sup_{f \in \mathcal{F}} U_{m-1}(h) + \sigma_m (\psi \circ f)(x_m) \right]
\]
\[
= (1 - \epsilon) \left[ \frac{1}{2} \sup_{f \in \mathcal{F}} U_{m-1}(f) + \sigma_m (\psi \circ f)(x_m) + \frac{1}{2} \sup_{f \in \mathcal{F}} U_{m-1}(f) - (\psi \circ f)(x_m) \right]
\]
\[
\leq \frac{1}{2} \left[ U_{m-1}(f_1) + (\psi \circ f_1)(x_m) \right] + \frac{1}{2} \left[ U_{m-1}(f_2) - (\psi \circ f_2)(x_m) \right]
\]

Let $n = \text{sgn}(f_1(x_m) - f_2(x_m))$, then by the $L$-Lipschitz property of $\psi$, we get:
\[
|(\psi \circ f_1)(x_m) - (\psi \circ f_2)(x_m)| \leq L |f_1(x_m) - f_2(x_m)|
\]
\[
= n L (f_1(x_m) - f_2(x_m))
\]

Thus, the inequality can be further expanded:
\[
(1 - \epsilon) \mathbb{E}_{\sigma_m} \left[ \sup_{f \in \mathcal{F}} U_{m-1}(f) + \sigma_m (\psi \circ f)(x_m) \right]
\]
\[
\leq \frac{1}{2} \left[ U_{m-1}(f_1) + U_{m-1}(f_2) + n L (f_1(x_m) - f_2(x_m)) \right]
\]
\[
= \frac{1}{2} \left[ U_{m-1}(f_1) + n L f_1(x_m) + U_{m-1}(f_2) - n L f_2(x_m) \right]
\]
\[
\leq \frac{1}{2} \sup_{f \in \mathcal{F}} \left[ U_{m-1}(f) + n L f(x_m) \right] + \frac{1}{2} \sup_{f \in \mathcal{F}} \left[ U_{m-1}(f) - nL f(x_m) \right]
\]
\[
=\underset{\sigma_m}{ \mathbb{E}} \left[ \sup_{f \in \mathcal{F}} U_{m-1}(f) + \sigma_m L f(x_m) \right]
\]

Since the above inequality holds for all $\epsilon > 0$, it must hold that:
\[
\underset{\sigma_m}{\mathbb{E}} \left[ \sup_{f \in \mathcal{F}} U_{m-1}(f) + \sigma_m (\psi \circ f)(x_m) \right] \leq \underset{\sigma_m}{\mathbb{E}}  \left[ \sup_{f \in \mathcal{F}} U_{m-1}(f) + \sigma_m L f(x_m) \right]
\]

For all $i = 1, \ldots, m-1$, using the above inequality, we get:
\[
\frac{1}{m} \underset{\sigma_1, \ldots, \sigma_m}{\mathbb{E}} \left[ \sup_{f \in \mathcal{F}} \sum_{i=1}^m \sigma_i (\psi \circ f)(x_i) \right]
\]
\[
\leq \frac{1}{m} \underset{\sigma_1, \ldots, \sigma_{m-1}}{\mathbb{E}} \left[ \underset{ \sigma_m}{\mathbb{E}} \left[ \sup_{f \in \mathcal{F}} U_{m-1}(f) + \sigma_m L f(x_m) \right] \right]
\]
\[
\leq \frac{1}{m} \underset{\sigma_1, \ldots, \sigma_{m-2}}{\mathbb{E}} \left[ \underset{\sigma_{m-1},\sigma_m}{\mathbb{E}} \left[ \sup_{f \in \mathcal{F}} U_{m-2}(f) + \sigma_{m-1} L f(x_{m-1}) + \sigma_m L f(x_m) \right] \right]
\]
\[
\vdots
\]
\[
\leq \frac{1}{m} \underset{\sigma_1, \ldots, \sigma_m}{\mathbb{E}} \left[ \sup_{f \in \mathcal{F}} \sigma_1 L f(x_1) + \sigma_2 L f(x_2) + \ldots + \sigma_m L f(x_m) \right]
\]
\[
= L \mathcal{R}_N(\mathcal{F})
\]
\end{proof}

\section{Proof of Lemma~\ref{lemma:rademacher-linear}}
\label{app:rademacher-linear}
Recall the Lemma~\ref{lemma:rademacher-linear}.
\begin{lemma}[Rademacher Complexity of Linear Operators \cite{bartlett2002rademacher}]
Consider the function class
\[
  \mathcal{H} \;=\; \Bigl\{\; h(W,x) \;=\; W\,\phi(x)\;\Big|\;\|W\|_F \,\le\, B\Bigr\},
\]
where $\phi:\mathcal{X}\to \mathbb{R}^r$ satisfies $\|\phi(x)\|\;\le\;C$ for all $x$. Then its Rademacher complexity is bounded by
\[
  \mathcal{R}_N(\mathcal{H}) \;\le\; \frac{B\,C}{\sqrt{N}}.
\]
\end{lemma}
\begin{proof}

Let $S=\{x_i\}_{i=1}^{N}$ be an i.i.d.\ sample drawn from $\mathcal{D}$, and let
$\sigma_1,\dots,\sigma_N\!\in\!\{-1,+1\}$ be independent Rademacher variables,
independent of $S$.
By definition,
\[
  \mathcal{R}_{N}(\mathcal{H})
  \;=\;
  \mathbb{E}_{S,\boldsymbol{\sigma}}
  \Bigl[
    \sup_{\|W\|_F\le B}
    \frac1N
    \sum_{i=1}^{N}
    \sigma_i\,W\,\phi(x_i)
  \Bigr].
\]

\paragraph{Step 1. Dual formulation.}
Write $v_S=\sum_{i=1}^{N}\sigma_i\phi(x_i)\in\mathbb{R}^{r}$.
Since $\langle W,v_S\rangle_F \!=\! \operatorname{tr}\!\bigl(W^\top v_S\bigr)$ and
$\sup_{\|W\|_F\le B}\!\langle W,v_S\rangle_F
   = B\,\|v_S\|$ (by Cauchy–Schwarz in Frobenius norm),  
\[
  \sup_{\|W\|_F\le B}
  \frac1N
  \sum_{i=1}^{N}\sigma_i\,W\,\phi(x_i)
  \;=\;
  \frac{B}{N}\,\|v_S\|.
\]

\paragraph{Step 2. Bring the expectation outside.}
\[
  \mathcal{R}_{N}(\mathcal{H})
  \;\le\;
  \frac{B}{N}\,
  \mathbb{E}_{S,\boldsymbol{\sigma}}\bigl[\|v_S\|\bigr].
\]

\paragraph{Step 3. Apply Jensen's inequality.}
\[
  \mathbb{E}\bigl[\|v_S\|\bigr]
  \;\le\;
  \sqrt{\mathbb{E}\bigl[\|v_S\|^{2}\bigr]}.
\]

\paragraph{Step 4. Expand the second moment.}
\begin{align*}
\mathbb{E}\bigl[\|v_S\|^{2}\bigr]
  &=
    \mathbb{E}_{S,\boldsymbol{\sigma}}
    \Bigl[
      \sum_{i,j=1}^{N}\sigma_i\sigma_j
      \,\langle \phi(x_i),\phi(x_j)\rangle
    \Bigr]                                           \\
  &=
    \sum_{i=1}^{N}
    \mathbb{E}_S\bigl[\|\phi(x_i)\|^{2}\bigr]
    \qquad(\text{since } \mathbb{E}[\sigma_i\sigma_j]=\delta_{ij})\\
  &\le
    N\,C^{2},
\end{align*}
because $\|\phi(x)\|\le C$ implies
$\mathbb{E}\|\phi(x)\|^{2}\le C^{2}$.

\paragraph{Step 5. Combine the bounds.}
\[
  \mathcal{R}_{N}(\mathcal{H})
  \;\le\;
  \frac{B}{N}\,
  \sqrt{N\,C^{2}}
  \;=\;
  \frac{B\,C}{\sqrt{N}}.
\]

\end{proof}

\section{Complete Proof of Theorem~\ref{thm:generalization-bound}}
\label{app:thm:generalization-bound}
Recall Theorem~\ref{thm:generalization-bound}.
\begin{theorem}[Generalization Bound for Implicit Knowledge Distillation]
Suppose the following conditions hold:
\begin{itemize}
    \item The teacher model is given by $f_T(x) \;=\; W^V x$ with $\|W^V x\| \,\le\, D$ for all $x$.
    \item The student model is $f_S(x;W) = W\,\phi\bigl(W^K x\bigr)$, where $\|W\|_F \,\le\, B$ and $\|\phi(W^K x)\| \,\le\, C$ for all $x$.  Here, $W^K$ is the fixed weight from LLM.
    \item The demonstration samples $\{x_i\}_{i=1}^N$ are drawn i.i.d.\ from $P_{X_D}$.
\end{itemize}
Then for any $\delta>0$, with probability at least $1-\delta$, the expected risk $L(W)$ and the empirical risk $\hat{L}(W)$ satisfy
\[
  L(W)
  \;\;\le\;\;
  \hat{L}(W)
  \;+\;
  \frac{4\,B\,C\;\bigl(D + B\,C\bigr)}{\sqrt{N}}
  \;+\;
  3\;\bigl(D + B\,C\bigr)^2 \,\sqrt{\frac{\log\!\bigl(2/\delta\bigr)}{2\,N}}.
\]
\end{theorem}
\begin{proof}
We divide the proof into five precise steps and keep every numerical
constant explicit.

\paragraph{Step 1: Define loss classes.}
For any weight matrix \(W\) and input \(x\), set
\[
  h(W,x) \;=\; f_T(x) \;-\; W\,\phi\!\bigl(W^K x\bigr), 
  \qquad
  \ell(W,x) \;=\; \bigl\|h(W,x)\bigr\|^{2}.
\]
Define
\[
  \mathcal{H} \;=\; \bigl\{\,h(W,\cdot)\,\bigm|\,\|W\|_{F}\le B\bigr\},
  \qquad
  \psi(z)=\|z\|^{2},
  \qquad
  \mathcal{L} = \psi \circ \mathcal{H}.
\]

Because \(\|f_T(x)\|\le D\) and
\(\bigl\|W\phi(W^Kx)\bigr\|\le \|W\|_{F}\,\|\phi(W^Kx)\|\le BC\),
we have
\[
  \|h(W,x)\|\;\le\;R:=D+BC,
  \qquad
  \ell(W,x)\;\le\;R^{2}.
\]

\paragraph{Step 2: \(\psi\) is Lipschitz on the ball \(\|z\|\le R\).}
For any \(z_{1},z_{2}\in\mathbb{R}^{d}\),
\[
  \bigl|\psi(z_{1})-\psi(z_{2})\bigr|
  \;=\;
  \bigl|\|z_{1}\|^{2}-\|z_{2}\|^{2}\bigr|
  \;=\;
  \bigl(\|z_{1}\|+\|z_{2}\|\bigr)\,
  \bigl|\|z_{1}\|-\|z_{2}\|\bigr|.
\]
By the reverse triangle inequality
\(\bigl|\|z_{1}\|-\|z_{2}\|\bigr|\le\|z_{1}-z_{2}\|\),
hence for \(\|z_{1}\|,\|z_{2}\|\le R\)
\[
  \bigl|\psi(z_{1})-\psi(z_{2})\bigr|
  \;\le\;
  2 R\,\|z_{1}-z_{2}\|.
\]
Thus \(\psi\) is \(L\)-Lipschitz on that region with
\(L = 2R = 2(D+BC)\).

\paragraph{Step 3: Rademacher complexity of \(\mathcal{L}\).}

\textit{3a.  Contraction.}

Talagrand’s contraction lemma (one‑sided) gives

\[
  \mathcal{R}_{N}(\mathcal{L})
  \;=\;
  \mathcal{R}_{N}(\psi\!\circ\!\mathcal{H})
  \;\le\;
  L \,\mathcal{R}_{N}(\mathcal{H})
  \;=\;
  2(D+BC)\,\mathcal{R}_{N}(\mathcal{H}).
\]

\textit{3b.  Bound \(\mathcal{R}_{N}(\mathcal{H})\).}

Let \(\sigma_{1},\dots,\sigma_{N}\) be Rademacher variables and
\(v_S = \sum_{i=1}^{N}\sigma_{i}\,\phi(W^K x_{i})\).
Then
\[
  \widehat{\mathcal{R}}_{S}(\mathcal{H})
  \;=\;
  \mathbb{E}_{\sigma}\!
  \Bigl[
    \sup_{\|W\|_{F}\le B}
    \frac{1}{N}\,
    \langle W,\;v_S\rangle_{F}
  \Bigr]
  \;=\;
  \frac{B}{N}\,\mathbb{E}_{\sigma}\!\bigl[\|v_S\|\bigr].
\]
Jensen’s inequality yields
\(
  \mathbb{E}_{\sigma}\|v_S\|
  \le
  \bigl(\mathbb{E}_{\sigma}\|v_S\|^{2}\bigr)^{1/2}.
\)
Since \(\mathbb{E}_{\sigma}[\sigma_i\sigma_j]=\delta_{ij}\),
\[
  \mathbb{E}_{\sigma}\|v_S\|^{2}
  \;=\;
  \sum_{i=1}^{N}\bigl\|\phi(W^K x_{i})\bigr\|^{2}
  \;\le\; N\,C^{2}.
\]
Therefore
\(
  \mathcal{R}_{N}(\mathcal{H})
  \le
  \dfrac{B\,C}{\sqrt{N}}.
\)

\textit{3c.  Combine.}

\[
  \mathcal{R}_{N}(\mathcal{L})
  \;\le\;
  \dfrac{2\,B\,C\,(D+BC)}{\sqrt{N}}.
\]

\paragraph{Step 4: Expected supremum gap.}
Set
\(
  \Phi(X_D)=\sup_{\|W\|_{F}\le B}\!\bigl(L(W)-\hat{L}(W)\bigr).
\)
Symmetrization implies
\[
  \mathbb{E}\!\bigl[\Phi(X_D)\bigr]
  \;\le\;
  2\,\mathcal{R}_{N}(\mathcal{L})
  \;\le\;
  \dfrac{4\,B\,C\,(D+BC)}{\sqrt{N}}.
\]

\paragraph{Step 5: Apply McDiarmid’s inequality.}

Replacing one sample \(x_{i}\) by an independent \(x_{i}'\)
changes \(\hat{L}(W)\) by at most \((D+BC)^{2}/N\).
Hence the function \(\Phi\) changes by at most
\(c = 2(D+BC)^{2}/N\).
Because \(\sum_{i=1}^{N} c^{\,2} = 4(D+BC)^{4}/N\),
McDiarmid’s inequality yields for any \(t>0\)
\[
  \Pr\!\bigl[\Phi(X_D)-\mathbb{E}\Phi(X_D)\ge t\bigr]
  \;\le\;
  \exp\!\Bigl(-\,\tfrac{N\,t^{2}}{2(D+BC)^{4}}\Bigr).
\]
Choose
\[
  t
  \;=\;
  3\,(D+BC)^{2}\,
  \sqrt{\dfrac{\log\!\bigl(2/\delta\bigr)}{2N}},
\]
so the right‑hand side equals \(\delta/2\).
With probability at least \(1-\delta\),
\[
  \Phi(X_D)
  \;\le\;
  \mathbb{E}\Phi(X_D) + t
  \;\le\;
  \frac{4\,B\,C\,(D+BC)}{\sqrt{N}}
  \;+\;
  3\,(D+BC)^{2}\,
  \sqrt{\frac{\log\!\bigl(2/\delta\bigr)}{2N}}.
\]

\paragraph{Step 6: Uniform bound for every \(W\).}
Since \(\Phi(X_D)\) dominates
\(L(W)-\hat{L}(W)\) for all \(\|W\|_{F}\le B\),
the displayed inequality yields, simultaneously for all such \(W\),
\[
  L(W)
  \;\le\;
  \hat{L}(W)
  \;+\;
  \frac{4\,B\,C\,(D+BC)}{\sqrt{N}}
  \;+\;
  3\,(D+BC)^{2}\,
  \sqrt{\frac{\log\!\bigl(2/\delta\bigr)}{2N}},
\]
with probability at least \(1-\delta\).
\end{proof}

\section{Complete Proof of Theorem~\ref{thm:offset_boundary_of_initial_weights}}
\label{app:thm:offset_boundary_of_initial_weights}
We give Theorem~\ref{thm:offset_boundary_of_initial_weights} without whitening here.

\begin{theorem}[Offset Boundary of the Initial Weights in the ICL-KD Process]
Suppose the following conditions hold:
\label{thm:app:offset_boundary_of_initial_weights}
\begin{itemize}
    \item Samples in target domain $\mathcal{D}$ satisfy $\|x\| \leq M_x$, 
        and the feature map $\phi(\cdot): \mathbb{R}^d \to \mathbb{R}^r$ is bounded: 
        $\|\phi(W^K x)\| \leq M_\phi$, where key transformation weights 
        $W^K \in \mathbb{R}^{k \times d}$ are fixed (frozen post-pretraining).
    \item Value transformation weights $W^V \in \mathbb{R}^{m \times d}$ satisfy 
        the Frobenius norm constraint $\|W^V\|_F \leq M_V$.
    \item The second-moment matrix 
        $\Sigma_\phi := \mathbb{E}_{x \sim \mathcal{D}}[\phi(W^K x)\phi(W^K x)^\top]$ 
        is invertible.
    \item Demonstration tokens $X_D=\{x_i\}_{i=1}^N$ are i.i.d.\ samples from some distribution $Q$.
    \item Define the one‐step distilled weight $\frac{\eta}{N}W^V X_D \phi(W^K x_D)^\top$ with learning rate $\eta>0$ (similar with $W_0$ in Eq.~\eqref{eq:W^*_and_W_0}).
    \item The maximum mean discrepancy between $\mathcal D$ and $Q$ is
        \[
          \mathrm{MMD}(\mathcal D,Q)
          :=\sup_{\|f\|_{\mathcal H}\le1}
            \Bigl|\mathbb E_{x\sim\mathcal D}[f(x)]
                 -\mathbb E_{x\sim Q}[f(x)]\Bigr|,
        \]
        where $\mathcal H$ is the RKHS induced by the kernel
                $\langle x\,\phi(W^Kx)^\top,\;x'\,\phi(W^Kx')^\top\rangle$.
\end{itemize}
Let the optimal reference weight be
\begin{equation}
  W^\star
  =\arg\min_W\;\mathbb E_{x\sim\mathcal D}\bigl\|W\phi(W^Kx)-W^Vx\bigr\|_2^2
  =\Sigma_\phi^{-1}\,
   \mathbb E_{x\sim\mathcal D}\bigl[W^Vx\,\phi(W^Kx)^\top\bigr].
\end{equation}
Then the expectation of the distilled weight and its deviation from $W^\star$ satisfy
\[
  \mathbb E[W_0]
  =\eta\,W^V\,\mathbb E_{x\sim Q}\bigl[x\,\phi(W^Kx)^\top\bigr],
\]
and, defining
\[
  \Delta W
  :=\mathbb E[W_0]-W^\star,
\]
we have the bound
\[
  \|\Delta W\|_{F}
  \;\le\;
  M_{V}M_{x}M_{\phi}\,
  \Bigl(
        \|\,\eta I-\Sigma_{\phi}^{-1}\|_{2}
        +\|\Sigma_{\phi}^{-1}\|_{2}\,
         \mathrm{MMD}\!\bigl(\mathcal D,Q\bigr)
  \Bigr).
\]
\end{theorem}

\begin{proof}[Proof.]
\noindent
\emph{Step 1: Expectation of the one‑step weight.}\;

Because $x_i\stackrel{\mathrm{i.i.d.}}{\sim}Q$, linearity of expectation gives  
\[
  \mathbb E[W_0]
  =\eta\,W^{V}\,
   \mathbb E_{x_i\sim Q}\!\Bigl[\tfrac1N\!\sum_{i=1}^N
        x_i\,\phi(W^{K}x_i)^{\!\top}\Bigr]
  =\eta\,W^{V}\,M_{Q},
\]
where $M_{Q}:=\mathbb E_{x\sim Q}[x\,\phi(W^{K}x)^{\!\top}].$

\noindent
\emph{Step 2: Optimal weight via normal equations.}\;

For  
\(
  \mathcal J(W)=\mathbb E_{\mathcal D}\|W\phi(W^{K}x)-W^{V}x\|_2^{2},
\)
setting $\nabla_W\mathcal J(W)=0$ yields  
\[
  W^{\star}\Sigma_{\phi}
  =W^{V}M_{D},
\]
where $\Sigma_{\phi}:=\mathbb E_{\mathcal D}[\phi\phi^{\!\top}]$ and $M_{D}:=\mathbb E_{x\sim\mathcal D}[x\,\phi(W^{K}x)^{\!\top}],$ so
\[
  W^{\star}
  =\Sigma_{\phi}^{-1}\,W^{V}M_{D}.
\]

\noindent
\emph{Step 3: Deviation bound without whitening.}

Define the offset  
\[
  \Delta W:=\mathbb E[W_{0}]-W^{\star}
           =\eta\,W^{V}M_{Q}-\Sigma_{\phi}^{-1}W^{V}M_{D}.
\]
Factorise as
\[
  \Delta W
  =(\eta I-\Sigma_{\phi}^{-1})\,W^{V}M_{Q}
   +\Sigma_{\phi}^{-1}W^{V}(M_{Q}-M_{D}).
\]

\emph{(i) Bounding the first term.}\;
Since $\|M_{Q}\|_{F}\le M_{x}M_{\phi}$,
\[
  \|(\eta I-\Sigma_{\phi}^{-1})\,W^{V}M_{Q}\|_{F}
  \le\|\,\eta I-\Sigma_{\phi}^{-1}\|_{2}\,
       \|W^{V}\|_{F}\,
       \|M_{Q}\|_{F}
  \le\|\,\eta I-\Sigma_{\phi}^{-1}\|_{2}\,
       M_{V}M_{x}M_{\phi}.
\]

\emph{(ii) Bounding the second term via MMD.}\;
Vectorise $A(x):=\mathrm{vec}\!\bigl(x\,\phi(W^{K}x)^{\!\top}\bigr)$.
Then  
\(
  \|M_{Q}-M_{D}\|_{F}
  =\|\mathbb E_{Q}A-\mathbb E_{\mathcal D}A\|_{2}.
\)
In the linear‑kernel RKHS  
$k(x,x')=\langle A(x),A(x')\rangle$,  
\[
  \|\mathbb E_{Q}A-\mathbb E_{\mathcal D}A\|_{2}
  \le M_{x}M_{\phi}\,
       \mathrm{MMD}\!\bigl(\mathcal D,Q\bigr).
\]
Hence
\[
  \|\Sigma_{\phi}^{-1}W^{V}(M_{Q}-M_{D})\|_{F}
  \le\|\Sigma_{\phi}^{-1}\|_{2}\,
       M_{V}M_{x}M_{\phi}\,
       \mathrm{MMD}\!\bigl(\mathcal D,Q\bigr).
\]

\medskip
\noindent
\emph{(iii) Combining.}\;
\[
  \boxed{
  \|\Delta W\|_{F}
  \;\le\;
  M_{V}M_{x}M_{\phi}\,
  \Bigl(
        \|\,\eta I-\Sigma_{\phi}^{-1}\|_{2}
        +\|\Sigma_{\phi}^{-1}\|_{2}\,
         \mathrm{MMD}\!\bigl(\mathcal D,Q\bigr)
  \Bigr).}
\]

Thus, we finished the proof.
\end{proof}

\section{Complete Proof of Theorem~\ref{cor:prompt_shift_risk_gap}}
\label{app:cor:prompt_shift_risk_gap}
Recall the Theorem~\ref{cor:prompt_shift_risk_gap}.

\begin{corollary}[Prompt--Shift Risk Gap]
Assume the conditions of Theorem~\ref{thm:offset_boundary_of_initial_weights}.
Let $Q_g$ (\emph{good prompt}) and $Q_b$ (\emph{bad prompt}) satisfy
\[
  \mathrm{MMD}_g := \mathrm{MMD}(\mathcal D,Q_g)
  <  \mathrm{MMD}_b := \mathrm{MMD}(\mathcal D,Q_b).
\]
Denote by $W_g,W_b$ the one-step distilled weights obtained from $Q_g,Q_b$ respectively, and define the KD risk on the target domain $\mathcal{D}$ as
\[
  \mathcal{L}_{KD}^{\mathcal D}(W)
  := \mathbb{E}_{x\sim\mathcal D}\!
     \bigl\|\,W\phi(W^Kx) - W^Vx \,\bigr\|_2^2 .
\]
Then
\begin{equation}
  \mathcal{L}_{KD}^{\mathcal D}(W_b) - \mathcal{L}_{KD}^{\mathcal D}(W_g)
  \;\le\;
  4\eta^2 M_T^2 M_\phi^4\bigl(\mathrm{MMD}_b-\mathrm{MMD}_g\bigr)^2  \;+\;
  4\eta M_T^2 M_\phi^2\bigl(1+2\eta M^2_\phi)(\mathrm{MMD}_b-\mathrm{MMD}_g)
\end{equation}
\end{corollary}
\begin{proof}

Recall the one-step KD rule
\[
   W(Q)=2\eta\,\mathbb{E}_{x\sim Q}
          \!\bigl[f_T(x)\,\phi(x)^{\top}\bigr].
\]
Define
\[
   A_g:=\mathbb{E}_{Q_g}[f_T(x)\phi(x)^{\top}],\qquad
   A_b:=\mathbb{E}_{Q_b}[f_T(x)\phi(x)^{\top}],
\]
so that
\[
   W_b-W_g=2\eta\,(A_b-A_g),\qquad
   \Delta_W:=\|W_b-W_g\|_F=2\eta\|A_b-A_g\|_F .
\]

\noindent
\emph{Step 1: Risk Decomposition.}
Let
\(u(x)=(W_b-W_g)\phi(x)\)
and
\(v(x)=f_T(x)-W_g\phi(x)\).
Then
\[
  \mathcal{L}_{KD}^{\mathcal D}(W_b)-\mathcal{L}_{KD}^{\mathcal D}(W_g)
  \;=\;
  \mathbb{E}_{x\sim\mathcal D}
  \bigl[\|u(x)\|^{2}-2\,v(x)^{\top}u(x)\bigr].
\]

\noindent
\emph{Step 2: bounding the two terms (exact constants).}

\begin{enumerate}
\item \(\|u(x)\|
       \le\|W_b-W_g\|_F\|\phi(x)\|
       \le\Delta_W M_\phi
       \;\Longrightarrow\;
       \|u(x)\|^{2}\le\Delta_W^{2}M_\phi^{2}.\)

\item
\(
  \|v(x)\|
  \le\|f_T(x)\|+\|W_g\|_F\|\phi(x)\|
  \le M_T + (2\eta M_T M_\phi)M_\phi
  =  M_T\bigl(1+2\eta M_\phi^{2}\bigr).
\)
Hence
\(
  |v^\top u|
  \le \|v(x)\|\|u(x)\|
  \le M_T M_\phi \bigl(1+2\eta M_\phi^{2}\bigr)\Delta_W .
\)
\end{enumerate}

Combining (i) and (ii) gives
\begin{equation}\label{eq:exact-intermediate}
  \mathcal{L}_{KD}^{\mathcal D}(W_b)-\mathcal{L}_{KD}^{\mathcal D}(W_g)
  \;\le\;
  \Delta_W^{2}M_\phi^{2}
  + 2M_T M_\phi\bigl(1+2\eta M_\phi^{2}\bigr)\Delta_W .
\end{equation}

\emph{Step 3: substitute \(\Delta_W\) via MMD.}

Theorem \ref{thm:offset_boundary_of_initial_weights} yields
\(
  \|A_b-A_g\|_F
  \le M_T M_\phi\bigl(\mathrm{MMD}_b-\mathrm{MMD}_g\bigr)
  =: M_T M_\phi\,\delta .
\)
Therefore
\(
  \Delta_W
  \le 2\eta M_T M_\phi\,\delta .
\)

Insert this into \eqref{eq:exact-intermediate}:
\[
\begin{aligned}
  \mathcal{L}_{KD}^{\mathcal D}(W_b)-\mathcal{L}_{KD}^{\mathcal D}(W_g)
  &\le 4\eta^{2}M_T^{2}M_\phi^{4}\,\delta^{2}
       +4\eta M_T^2 M_\phi^{2}
        \bigl(1+2\eta M_\phi^{2}\bigr)\delta \\
\end{aligned}
\]

\noindent
\textbf{Remark.}  
This exact bound differs from the main-text presentation only by explicit constants; it leaves all qualitative conclusions in the paper unchanged.
\end{proof}


\end{document}